\def\year{2022}\relax
\documentclass[letterpaper]{article} 
\usepackage{aaai23} 
\usepackage{times}  
\usepackage{helvet}  
\usepackage{courier}  
\usepackage[hyphens]{url}  
\usepackage{algorithm}
\usepackage{algorithmic}
\usepackage{newfloat}
\usepackage{listings}
\usepackage{color}
\usepackage{newfloat}
\usepackage{listings}
\usepackage{graphicx} 
\usepackage{cuted}
\usepackage{natbib}
\usepackage{caption}
\usepackage{color}
\usepackage{natbib}  
\usepackage{caption} 
\usepackage{algorithm}
\usepackage{algorithmic}

\usepackage{newfloat}
\usepackage{listings}
\floatstyle{ruled}
\newfloat{listing}{tb}{lst}{}
\floatname{listing}{Listing}
\nocopyright
\usepackage{times}  
\usepackage{helvet} 
\usepackage{courier}  
\usepackage[hyphens]{url}  
\usepackage{graphicx} 
\usepackage{caption} 
\usepackage{subca ption}
\usepackage{cite}
\usepackage{textcomp}
\usepackage{xcolor}
 \usepackage{mathrsfs}
\usepackage{amsmath,amssymb,amsfonts}
\usepackage{makecell}
\usepackage{graphicx}
\usepackage{textcomp}
\usepackage{xcolor}
\usepackage{multicol}
\usepackage{multirow}
\usepackage{url}
\usepackage{amsthm} %
\usepackage{algorithm}
\usepackage{algorithmic}
\usepackage{graphicx}
\usepackage{textcomp}
\usepackage{xcolor}
\usepackage{color}
\usepackage{booktabs}
\usepackage[switch]{lineno}
 \title{Exploring Optimal Substructure for Out-of-distribution Generalization via Feature-targeted Model Pruning}
\author{
    Yingchun Wang\textsuperscript{1,2}, 
    Jingcai Guo\textsuperscript{1}, 
    Song Guo\textsuperscript{1}, 
    Weizhan Zhang\textsuperscript{2}, 
    Jie Zhang \textsuperscript{2}
}
\affiliations{
    \textsuperscript{1}Department of Computing, The Hong Kong Polytechnic University, Hong Kong SAR, China\\
    \textsuperscript{2} Department of Computer Science and Technology, Xi'an Jiaotong University, Xi'an, China\\
    daenerys.wang@connect.polyu.hk, 
    \{jc-jingcai.guo, song.guo\}@polyu.edu.hk, \\
    zhangwzh@xjtu.edu.cn, 
    jie-comp.zhang@polyu.edu.hk
}
\usepackage{bibentry}
\newtheorem{lemma}{Lemma}
\newtheorem{definition}{Definition}
\newtheorem{proposition}{Proposition}

\newtheorem{corollary}{Corollary}
\begin{document}
\maketitle
\begin{abstract}
Recent studies show that even highly biased dense networks contain an unbiased substructure that can achieve better out-of-distribution (OOD) generalization than the original model.
Existing works usually search the invariant subnetwork using modular risk minimization (MRM) with out-domain data.
Such a paradigm may bring about two potential weaknesses: 1) \textit{Unfairness}, due to the insufficient observation of out-domain data during training; and 2) \textit{Sub-optimal OOD generalization}, due to the feature-untargeted model pruning on the whole data distribution.
In this paper, we propose a novel \textbf{\underline{S}}purious \textbf{\underline{F}}eature-targeted model \textbf{\underline{P}}runing framework, dubbed SFP, to automatically explore invariant substructures without referring to the above weaknesses. 
Specifically, SFP identifies in-distribution (ID) features during training using our theoretically verified task loss, upon which, SFP can perform ID targeted-model pruning that removes branches with strong dependencies on ID features. 
Notably, by attenuating the projections of spurious features into model space, SFP can push the model learning toward invariant features and pull that out of environmental features, devising optimal OOD generalization. 
Moreover, we also conduct detailed theoretical analysis to provide the rationality guarantee and a proof framework for OOD structures via model sparsity, and for the first time, reveal how a highly biased data distribution affects the model's OOD generalization. 
Extensive experiments on various OOD datasets show that SFP can significantly outperform both structure-based and non-structure OOD generalization SOTAs, with accuracy improvement up to 4.72\% and 23.35\%, respectively.

\end{abstract}


\section{Introduction}
The reliance on intelligent applications has kept increasing recently since machine learning has demonstrated its excellent capabilities in various fields such as computer vision, natural language processing, recommendation system, etc.~\cite{jordan2015machine}. 
%
However, when faced with the fickle data distribution in the real world, most applications, born under the ideal assumption that the data are identical and independently distributed (I.I.D.), can expose their vulnerability. Therefore, it is crucial to improve the out-of-distribution (OOD) generalization ability of these models. 
Formally, OOD generalization aims to adapt the learned knowledge to unknown test data distribution. 
In recent years, a wealth of literatures have been generated for this field. 
%
\cite{survey} summarizes several popular branches under the supervised setting, including domain generalization, causal invariant learning, and stable learning. 
Specifically, domain generalization (DG)~\cite{dg1,dg2} combines multiple source domains to learn models that generalize well on unseen target domains. 
Differently, causal learning and invariant learning~\cite{cl1,iv1,iv2} explore the invariance of data predictions in a more principled way for causal inference. 
In anther way, stable learning \cite{st1,st2} aims to establish a consensus between causal inference and machine learning to improve the robustness and credibility of models.

Most recently, some works address OOD problem from the perspective of model structure, which is also our focus. Compared to the above-mentioned methods, the model-structure approach has an extra advantage that it is general and can be embedded in most SOTAs and further improve their performance. 
For example, \cite{ar1} provides sufficient and intuitive motivation for this branch of OOD generalization, suggests that over-parameterized learning models could degrade OOD performance through data memorization and overfitting. Differently, \cite{can} claims that even highly spurious features-related full networks can contain particular substructures that may achieve better OOD generalization performance than the full network, and propose a module detection technique, with the guidance of OOD data, to identify this functional lottery. 

Despite the progress made, the model-structure methods are mostly empirically based, lacks theoretical explanations and proofs of method effectiveness. Nevertheless, they apply existing non-OOD-specific techniques such as network architecture search and module detection to find OOD lottery tickets, which may degrade the effectiveness of these techniques in OOD setting. For example, \cite{can} also mentioned that the sparsity of the weights is not exactly the sparsity of the model about spurious features in their method. 
And most of them rely on the guidance of fully exposed OOD data or the assumption of known causal structure, which is highly unlikely to be feasible in practice, inducing additional post-training cost.

In this paper, we propose a novel automatic model trimming method to automatically and progressively find out the optimal invariant substructure for better OOD generalization. Our method automatically recognizes ID features with high probability during training, thus releasing the dependence on artificial preliminaries.
Immediately after, SFP prevents the model from fitting spurious features as much as possible, and executes model sparse particularly for ID feature sparse. The whole learning pipeline is in one-shot training.
Specifically, SFP build two subspaces spanning from highly biased training data to provide the coordinate basis for spurious and invariant features respectively, and a model feature space as the reference for feature projection. 
We prove that inputs with smaller prediction loss contain more spurious features during training, as the rationale for identifying spurious features. 
Through weakening the feature projections only from 
those identified ones into model space, we increase the resistance of the model space basis to learn towards directions of the subspace spanning from spurious features. 
Thus, SFP progressively adjusts the component rank in projection matrix to the ordering of invariant feature correlation via singular value decomposition(SVD). The directions of the model feature space corresponding to the lowest singular values are sparsed out. 
We provide both theoretical analysis and experimental evidence to prove the effectiveness of SFP. \\
In summary, our contributions can be listed as follows:
\begin{itemize}
    \item We analytically prove the rationality and effectiveness of finding better OOD structures through model sparsity. We provide a proof framework for sparse models for OOD, making up for the lack of theoretical guidance from previous work in this branch.
    \item We propose methods to automatically find OOD substructures during training without prior causal assumptions and additional OOD datasets.
    \item Our extensive experiments show that our SFP can successfully sparse in-distribution feature projection and achieve ${93.42}\%$ test accuracy on full-colored-mnist dataset, only $0.62\%$ behind the upper-bond performance.
\end{itemize}

\section{Related Work}
\textbf{Out-of-distribution.}
\label{ood related work}
When moving a machine learning prediction model from the laboratory to the real world, a distributed shift between test data and training data is usually unavoidable. A series of efforts to improve OOD 
generalization ability in recent years can be roughly divided into three parts: unsupervised representation learning, supervised model learning and optimization.
Unsupervised representational learning uses human prior knowledge to design and limit the process of representational learning so that the learned representations have some attributes that may contribute to OOD generalization.
The purpose of disentangled representation learning is to learn to separate the representation of different and informative variables in data \cite{DBLP:journals/pami/BengioCV13, DBLP:conf/iclr/LocatelloBLRGSB19, DBLP:journals/pieee/ScholkopfLBKKGB21}.
Causal representation learning \cite{DBLP:conf/cvpr/YangLCSHW21, DBLP:journals/corr/abs-2010-02637, DBLP:journals/pieee/ScholkopfLBKKGB21} aims to learn variables in causal graphs in an unsupervised or semi-supervised manner. Causal representations can be regarded as the ultimate goal of disentanglement, satisfying the informal definitions of interpretability and sparsity of disentangled representations.
As we mentioned in the previous introduction, typical supervised model learning for OOD generalization includes domain generalization \cite{dg1,dg2,DBLP:conf/iclr/DuZ0S21}, causal  learning \cite{cl1,iv1,iv2} and stable learning \cite{st1,st2}.
There are also some other optimization methods recently like distributed robust optimization \cite{DBLP:journals/mp/BertsimasGK18} or the immutable optimization method \cite{DBLP:conf/icml/LiuH00S21}.\\
\noindent\textbf{Model Pruning.}
\label{sec: pruning related work}
The purpose of network pruning is to eliminate unnecessary weights from over-parameterized neural networks.
Optimal Brain Damage \cite{DBLP:conf/nips/CunDS89} removes weight parameters based on the Hessian matrix of objective function and fine-tunes the updated network.
Early pruning \cite{DBLP:conf/nips/HanPTD15} began with the removal of weights or nodes with small-norm from the DNN.
However, this kind of unstructured pruning actually improves reasoning time very little without specialized hardware \cite{DBLP:conf/nips/WenWWCL16}.
Therefore, channel pruning or filter pruning \cite{DBLP:conf/nips/WenWWCL16, DBLP:conf/iclr/0022KDSG17} has gradually become the mainstream in this field. 
For example, \cite{DBLP:conf/ijcai/HeKDFY18} resets less important filters at every epoch while updating all filters; \cite{DBLP:conf/cvpr/ZhaoNZZZT19} uses stochastic variational inference to indicate and remove the channels with smaller mean and variance.
Briefly, previous approaches on searching for model substructures with better OOD generalization performance essentially follow the traditional empirical risk-guided model pruning paradigm. This feature-untargeted model sparsity makes previous methods only obtain suboptimal OOD models.

\section{Proposed Method}
In this section, 
we formalize the model structure-based OOD problem in a complete inner product space, and provide a theoretical analysis to investigate the impact of ID data and OOD data on model performance.
Based on this framework, in the following part,
we elaborate the optimization objective of SFP and theoretically demonstrate its effectiveness. 
\subsection{Notations and Preliminaries}\label{sec:3.1}
\subsubsection{Linear Parameterized Notations}
Define $X_{id}\in \mathbb{R}^{p \times d}$, $X_{ood} \in \mathbb{R}^{q \times d}$ are datasets from in-distribution and out-distribution domain respectively, where $p$, $q$ represent the corresponding data numbers, and $d$ is the feature dimension. Thus, the training dataset could be denoted as $X=X_{id}\cup X_{ood}$. In the problem setting of large OOD shift, $p \gg q$. Given $p_i$ and $p_o$ as the proportion of ID samples and OOD  samples in the training dataset, respectively, we have $p_i \gg p_o$ and $p_i + p_o = 1$. Use $\mathcal{W}\in \mathbb{R}^{m\times d}$ to represent the parameters of the feature extractor in convolution neural networks (CNN), where $m$ is the dimension of the feature map output before the classification layer. 

To rebuild the problem of OOD generalization in a complete inner product space, some auxiliary notations are defined as follows: 
Let $R = \boldsymbol{C}(\mathcal{W}^\top)$, $S = \boldsymbol{C}(X_{id}^\top)$ and $U = \boldsymbol{C}(X_{ood}^\top)$ 
be the subspace spanning the parameterized model, ID data and OOD data by their rows, respectively. 
Let $E \in \mathbb{R}^{d \times \text{dim}(R)}$, $F \in \mathbb{R}^{d \times \text{dim}(S)}$ and $G \in \mathbb{R}^{d \times \text{dim}(U)}$ be the basis of orthogonal matrix with standard columns and rows $R$, $S$ and $U$ respectively. 
Then, the original algebraic representation of the model and dataset can be reformulated in linear form as spanning spaces over a set of learnable basis vectors. We propose the following proposition and analyze the details of the linear transformation as follows:
\begin{proposition}
Model substructures and the feature representations can be effectively corresponded in linear form by the singular value decomposition of the feature projections in the model space.
\label{proposition1}
\end{proposition}
\noindent \textbf{Discussion (Model):}
Define $E^\top F \in \mathbb{R}^{\text{dim}(R) \times \text{dim}(S)}$ as the basis of $\boldsymbol{C}(X_{id} \mathcal{W}^\top)$ spanning the ID feature projections. 
Similarly, $E^\top G \in \mathbb{R}^{\text{dim}(R) \times \text{dim}(U)}$ is the basis of $\boldsymbol{C}(X_{ood} \mathcal{W}^\top)$ spanning the OOD feature projections. 
For $ a \in \boldsymbol{C}(E^\top F)$ which is a column vector of $\mathbb{R}^{\text{dim}(S)}$, we could define the projection feature of an ID data on model space as $r_1=E^\top Fa$. Equally, for $ b \in \boldsymbol{C}(E^\top G)$, which is a column vector of $\mathbb{R}^{\text{dim}(U)}$, the projection feature of a OOD data on model space is $r_2=E^\top Gb$. 
Then, the projection feature vectors of training dataset on model space could be defined as $r=p_ir_1+p_or_2$. Multiply the basis of  model space $E$ with the target vector $r$ to obtain a set of coordinates of output feature maps on the model space. Then, we could donate the set of model parameters $\mathcal{W}$ as $Er$. Assume $\mathcal{W}^*$ is the optimal set of model parameters, $\mathcal{W^*}=Er^*$, where $r^*= p_iE^\top Fa^* + p_ob*$ and $a^*$, $b^*$ would be the true feature projections.\\
\noindent \textbf{Discussion (Data):} 
In $S=\boldsymbol{C}(X_{id}^\top)$ with basis $F$ spanning $X_{id}$, 
$\forall~x_i \in X_{id}$, $\exists~z \in \mathbb{R}^{\text{dim}(S)}$, $x = (Fz)^\top$. $X_{id}=(FZ)^\top$, where $Z=\{z\}$. 
Similarly, in $U=\boldsymbol{C}(X_{ood}^\top)$ with basis $G$, $\forall~x_{ood} \in X_{ood}$, $\exists~v \in \mathbb{R}^{\text{dim}(U)}$, $x_{ood} = (Gv)^\top $. $X_{ood}=(GV)^\top$, where $V=\{v\}$. \par
\subsubsection{Preliminaries Optimization Target.}
\begin{definition}
Equal model learning on the whole data distribution guided only on risk minimization is defined as undirected learning.
\end{definition}
\begin{definition}
Trained independently from scratch for the same number of iterations, the substructure in the original model with the best OOD generalization performance is defined as the OOD lottery~\cite{can}.
\end{definition}

\noindent For structure-based preliminaries searching the OOD lottery based on undirected learning, the optimization target can be formulated as follows:
\begin{equation}
    \min~ \mathcal{L}(\mathcal{W}, X, Y) = E_{X}\left \| X \mathcal{W} -Y \right \|_2^2 + S\left( \mathcal{W}\right),
    \label{1}
\end{equation}
where $\mathcal{L}$ is the task-dependent loss function. $S$ is the function that induces the sparsity of the model structure to find the target subnetwork. $X \in \mathbb{R}^{n \times d}$ represents the set of training samples, where $n=p+q$ and $Y$ represents the corresponding groudtruth of the feature (sample) projections into the model space.

At each iteration $t$, the task loss with respect to $\mathcal{W}_t$ could be calculated as:
\begin{equation}
\setlength{\abovedisplayskip}{1pt}
\setlength{\belowdisplayskip}{1pt}
 \begin{aligned}
    \mathcal{L}_t&=\left \| X \mathcal{W}_t -Y \right \|_2^2=\left \| X \mathcal{W}_t -X \mathcal{W}^* \right \|_2^2,
\end{aligned}
\label{2}
\end{equation}
and the gradients accumulation could be formulated as:
\begin{equation}
    \sum_{t=1}^{\infty } \frac{\partial \mathcal{L}_t}{\partial \mathcal{W}_t} 
    = \sum_{t=1}^{\infty }2(\mathcal{W}_t X-\mathcal{W}^* X)^\top X.
\label{3}
\end{equation}
Regarding the orthogonal basis of the model space as the dominant left singular vectors
provides a consistent orthonormal basis for the training process. The right singular vectors correspond to the input data features, and the corresponding singular values can be defined as indicators of the importance of the current data features to the model structure.
To internally observe the impact of ID and OOD features on the model during training, the gradient accumulation is further transformed into linear form as follows: 
\begin{equation}
\setlength{\abovedisplayskip}{3pt}
\begin{aligned}
    \sum_{t=1}^{\infty } \frac{\partial \mathcal{L}_t}{\partial \mathcal{W}_t} 
    %
    %
    = 2\sum_{t=1}^{\infty }\left\{  p_i^2\widetilde{a_t} \Sigma_{E^\top F}^2 X_{id} 
    + p_o^2 \widetilde{b_t} \Sigma_{E^\top G}^2X_{ood} \right\}&\\
    ~~~+ p_i p_o \left\{\widetilde{a_t} \Sigma_{F^\top E} \Sigma_{E^\top G} X_{ood} + \widetilde{b_t} \Sigma_{G^\top E}\Sigma_{E^\top F} X_{id}\right\}&,
\end{aligned}
\label{4}
\end{equation}
where $\widetilde{a_t}= a_t- a^*~ \in \boldsymbol{C}(E^\top F)$. And $s \in S$, $u \in U$, $S^*$ represents the optimal feature columns in $S$. $\mathcal{W}^*$ is the optimal OOD submodel.
Note that $\widetilde{b_t}= b_t - b^*~ \in \boldsymbol{C}(E^\top F)$.
$\Sigma$ represents the corresponding singular value matrix, and for brevity, we omit the $t$-index of the singular value matrix of the feature projection in the context. Since $\text{dim}(U)=q \ll \text{dim}(S)=p$, we have $\min~\Sigma_{F^\top G} = \min~\Sigma_{G^\top F} = \sigma_{G^\top F}^q$. Similarly, $\min~\Sigma_{F^\top E} = \min~\Sigma_{E^\top F} = \sigma_{E^\top F}^m$, $\min~\Sigma_{G^\top E} = \min~\Sigma_{E^\top G} = \sigma_{E^\top G}^m$. 
Finally, the model parameters could be calculated as in Eq.~\ref{5}.
\begin{equation}   
\setlength{\abovedisplayskip}{1pt}
\setlength{\belowdisplayskip}{1pt}
\begin{aligned}
    &\min W^\infty  
    %
    %
    =W_0 - 2lr\sum_{t=1}^{\infty }\sum_{i=1}^{m} p_i^2\widetilde{a}_t\sigma_{E^\top F,i}^2 X_{id} \\
    &~- p_o^2 \widetilde{b}_t \sigma_{E^\top G,i}^2 X_{ood} - p_i p_o \sigma^i_{F^\top E}\sigma^i_{E^\top G} (\widetilde{a_t} X_{ood} +\widetilde{b}_t X_{id}).
\end{aligned}  
\label{5}
\end{equation}
\subsubsection{Biased Performance on OOD and ID Data}
Based on the gradient flow trajectories, we compare the learning process and final performance of the models for spurious and invariant features, respectively. We capture a fact that: the model structure obtained by undirected learning has obvious performance difference between ID data and OOD data. With this observation, we propose the following propositions.
\begin{proposition}
Undirected learning (full or sparse training) on biased data distributions can lead to significantly different forward speed of the model learning along different data feature directions, and the difference has a second-order relationship with the proportion of different data distributions in the training set.
\begin{equation}   \setlength{\abovedisplayskip}{1pt}
\setlength{\belowdisplayskip}{1pt}
\begin{aligned}
   \left|\frac{\partial W_t}{\partial \widetilde{a_t}} - \frac{\partial W_t}{\partial \widetilde{b_t}}\right| \approx 2(p_i^2 \Sigma_{E^\top F}^2 - p_o^2 \Sigma_{E^\top G}^2 ).
\end{aligned}  
\label{6}
\end{equation}
\label{proposition2}
\end{proposition}
\noindent\textbf{Discussion (Update gradient):}  We compute the direction gradients along the directions of the feature projections of ID and OOD data, respectively. As shown in Eq.~\ref{6}, with $p_i>p_o$, the learning of the basis of the model space is gradually biased towards the directions of spurious features. By performing singular value decomposition on the projection of the basis vector of the feature space through the model space, the obtained singular value matrix can be regarded as the fitting degree of the model on the corresponding data distribution at $t_{th}$ iteration.
\begin{proposition}
Undirected learning (full or sparse training) on biased data distributions causes the model to be more biased towards training features with a larger proportion, bringing significant performance differences in different data distributions. 
%
%
\begin{equation}   
\begin{aligned}
\mathcal{L}_{ood} &-\mathcal{L}_{id}
\approx(p_i^2-p_{o}^2)(1-\Sigma_{F^\top G}) + \epsilon >0,
\end{aligned} 
\label{7}
\end{equation}
where $\epsilon$ is the difference of initial feature projections between ID and OOD data due to model initialization error.
\label{proposition3}
\end{proposition}
\begin{corollary}
When searching for OOD structures with higher generalization performance, undirected learning of models in non-ideal training domains (with spurious features) can only lead to sub-optimal OOD structures.
\end{corollary}
\noindent\textbf{Discussion (Performance difference):} 
The result intuitively shows that the undirectly learned model performs better on feature distributions with larger sample numbers.
As shown in Eq.~\ref{7}, the difference in model performance between OOD and ID data is linearly related to the proportion of the corresponding samples and the correlation degree between the different feature distributions. 
What's more, when the out-of-domain data has the same proportion as in-domain data in the training dataset $(p_i=p_o)$, 
or the data distributions of OOD are consistent with ID, the task loss difference between OOD and ID data could be reduced to zero.


\subsection{SFP: An ID-targeted Model Sparse Method}{\label{sec:3.2}}
In order to solve the problem of sub-optimal OOD structure caused by undirected training, 
we proposed an ID-targeted model sparsity method to effectively remove model branches that are only strongly correlated with spurious features.
As shown in Figure.~\ref{f1}, the training pipeline of SFP consists of two following stages including spurious feature identification and model sparse training. Specifically, SFP identifies in-domain data with large spurious feature components with high probability by observing the model loss during training, and then performs spurious features-targeted model sparsity by analyzing the singular value decomposition of the feature projection matrix between the data and model space. We also provide a detailed theoretical analysis on both stages of SFP in the following part. 
\begin{figure*}
\centering
\setlength{\abovecaptionskip}{2pt}
\includegraphics[width=0.88\textwidth]{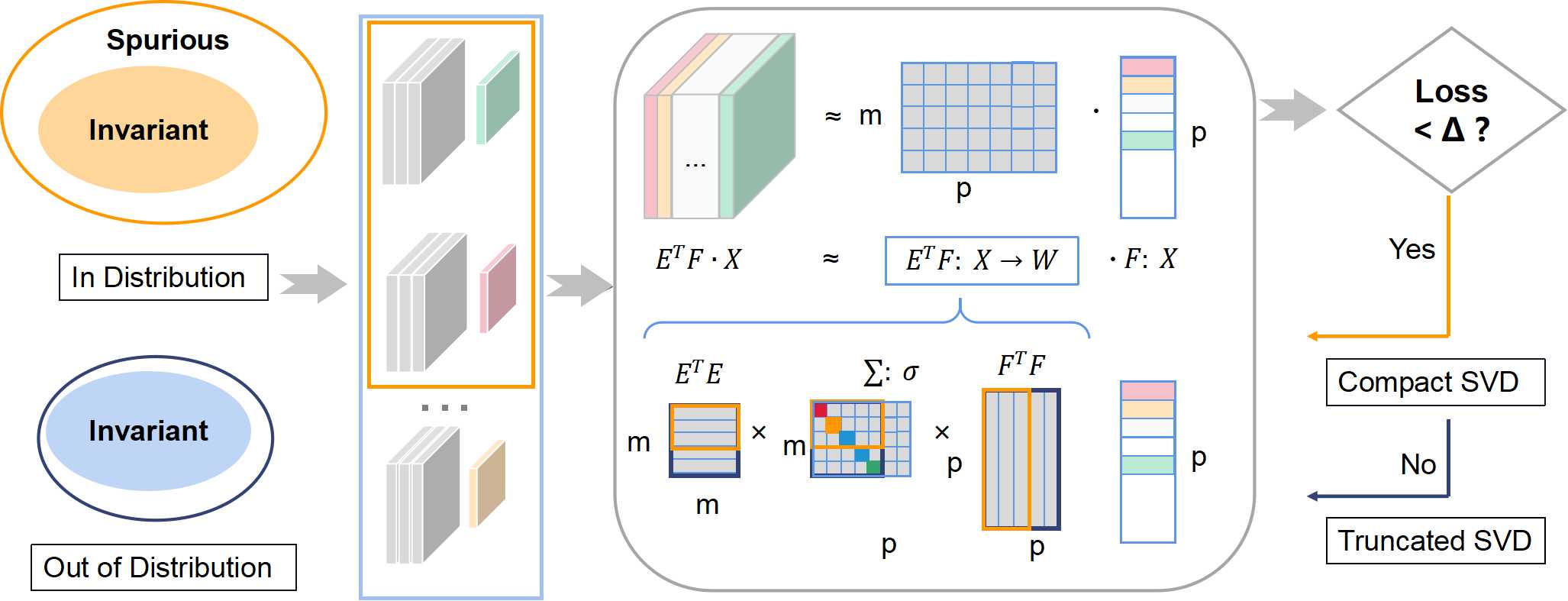}
\caption{The training pipeline of SFP.}
\label{f1}
\end{figure*}

\subsubsection{Spurious feature identification}

As shown in Proposition.\ref{proposition3}, 
if no intervention is applied, a model trained on a highly biased data distribution would be gradually biased towards ID data with lower prediction losses. Since the loss difference between ID and OOD data can be approximately computed with $(p_i^2-p_o^2)(1-\sigma_{F^\top G})$, it is adopted as the identification criterion for spurious features in each iteration. 
Simply put, if the loss corresponding to the current data is lower than the threshold $\Delta$, the model determines that the current data is likely to be an ID sample dominated by spurious features, and further prunes the spanning sets of model space along the directions of spurious feature projections. To compute $\Delta$, we first investigate the average loss in the $t-1$-th iteration as follows:
\begin{equation}
    \begin{aligned}
    \overline{\mathcal{L}^{t-1}}\approx \mathcal{L}^{t-1}_{id} + p_o(p_i-p_o)(1-\sigma_{F\top G}^{t-1}),
    \end{aligned}
\label{8}
\end{equation}
without any sparse, the $\Delta$ should be:
\begin{equation}
    \begin{aligned}
    d^t=|\mathcal{L}^{t}_{id}- \overline{\mathcal{L}^{t-1}}|\approx p_o(p_i-p_o)(1-\sigma_{F\top G}^{t-1}).
    \end{aligned}
\label{9}
\end{equation}

However, if model space have been sparse for spurious features at round $t-1$, $\mathcal{L}^{t}_{id}$ would be closer to $\overline{\mathcal{L}^{t-1}}$, and the ideal $\Delta^t_{sparse}$ should be:
\begin{equation}
    \begin{aligned}
    d^t_{sparse} =|\mathcal{L}^{t}_{id}- \overline{\mathcal{L}^{t-1}}| \approx p_o(p_i-p_o)(1-\sigma_{F^\top G}^t ),
    \end{aligned}
\label{10}
\end{equation}
where $\sigma_{F^\top G}^t$ is:
\begin{equation}
    \sigma_{F^\top G}^t=\arg \min_{\Sigma_{F^\top G}}~ \sigma_{F\top G}^t-\sigma_{F^\top G}^{t-1} >0,
\label{11}
\end{equation}
and $\mathcal{L}_{id}^t$ are highly likely to locate in interval  $\left[|\overline{\mathcal{L}^{t-1}}|-d^t,|\overline{\mathcal{L}^{t-1}}|-d^t_{sparse}\right]$. The upper bound is used as the threshold $\Delta$ for spurious feature identification.
Spurious features identification lays the foundation for hindering model from fitting spurious features by trimming the corresponding substructure. 

\subsubsection{ID-targeted model sparse}
Once perceiving spurious feature-dominated samples, SFP reacts by weakening the projections of spurious features into model space to prevent the model from over-fitting with spurious features. 

To analysis the projections into model space, we define $\Xi \in \mathbb{R}^{m \times m}$, $\Lambda \in \mathbb{R}^{p \times p}$, $\Gamma \in \mathbb{R}^{q\times q}$ as the normalized orthogonal basis of the  $\boldsymbol{C}(E^\top E)$, $\boldsymbol{C}(E^\top F)$ and $\boldsymbol{C}(E^\top G)$, spanning the model projections, the ID projections and the OOD projections, respectively. $\xi_i$, $\lambda_i$ and $\gamma_i$ represent the $i_th$ column vectors in $\Xi$, $\Lambda$ and $\Gamma$ respectively. 
\begin{lemma}
ID features targeted model sparse can effectively reduce the performance deviation of the learned model between ID data and OOD data. 
\begin{equation}
\begin{aligned}
  &\left|\frac{\mathcal{R}(X_{ood})-\mathcal{R}(X_{id})}{\mathcal{R}(X_{ood})^{sparse}-\mathcal{R}(X_{id})^{sparse}}\right|\\
  &\approx \left|\frac{\sum_{j=1}^{m}  p_{o} \tilde{\sigma_j} \xi_j \gamma_j X_{ood}- \sum_{i=1}^{m} p_{i} \sigma_i \xi_i \lambda_i X_{id}}{\sum_{j=1}^{m}  p_{o} \tilde{\sigma_j} \xi_j \gamma_j X_{ood}-\sum_{i=1}^{\vartheta}p_{i} \sigma_i \xi_i \lambda_i X_{id}}\right| \ge 1,
\end{aligned}
\label{12}
\end{equation}
where $\mathcal{R}(\cdot)$ represents the empirical risk function, $\sigma_{i}, \tilde{\sigma}_i$ is the $i$-th maximums in $\Sigma_{E^\top F}$ and $\Sigma_{E^\top G}$. And we have $\sigma>0$ since the singular values are non-negative. $m$ and $\vartheta$ are the rank of the singular value matrix after performing compact singular decomposition and truncated singular value decomposition on the projections, respectively.  
\label{lemma1}
\end{lemma}

\begin{proof}[Proof of \textbf{Lemma.\ref{lemma1}}]
As mentioned before, the projection space before model sparse could be represented as:
    \begin{equation}
\setlength{\abovedisplayskip}{1pt}
\setlength{\belowdisplayskip}{1pt}
        \begin{aligned}
        Er 
        = \sum_{i=1}^{m}\left ( p_i \sigma_i \xi_i \lambda_i ^\top + p_o \tilde{\sigma_i} \xi_i \gamma_i ^\top \right ) 
        \end{aligned}
    \label{13}
    \end{equation}
Specifically, SFP first perform singular value decomposition on the feature projections which maps input data to a set of coordinates based on the orthonormal basis of model space.
The matrices of left and right singular vectors correspond to the standard orthonormal basis of the model space and data space, respectively. The matrix of singular values correspond to the direction weight of the action vectors in the projection matrix. SFP prunes the model by trimming the smallest singular values in $\Sigma$ as well as their corresponding left and right singular vectors. 
In this way, SFP could remove the spurious features in ID data space and substructures in the model space simultaneously in a targeted manner along the directions with weaker actions for projection. 
and the projection space after sparse with only the most important $\vartheta$ singular values can be formalized as:
    \begin{equation}
        \begin{aligned}
        Er^{sparse} &= p_i \Xi \Sigma_{E^\top F} \Lambda^{-1} + p_o \xi \Sigma_{E^\top G} \Gamma^{-1}\\
        &= \sum_{i=1}^{\vartheta} p_i \sigma_i \xi_i \lambda_i ^\top + \sum_{j=1}^{m} p_o \tilde{\sigma_j} \xi_j \gamma_j ^\top. 
        \end{aligned}
    \label{14}
    \end{equation}
Based on the representation of the projection spaces, the model response to data features $\mathcal{R}(X)=ErX$ can be calculated as:
\begin{equation}
        \begin{aligned}
        \mathcal{R}(X) &= \left\{ p_i \Xi \Sigma_{E^\top F} \Lambda^{-1} + p_o \xi \Sigma_{E^\top G} \Gamma^{-1}\right\}^\top X^\top\\
        &= \sum_{i=1}^{m} \left\{p_i \sigma_i \xi_i \lambda_i ^\top X^\top + p_o \tilde{\sigma_i} \xi_i \gamma_i ^\top X^\top\right\} 
        \end{aligned}
    \label{15}
    \end{equation}
\end{proof}
\subsubsection{Discussion:} 
Model sparsity for ID features is achieved by a low-rank approximation of $\vartheta$-order to the decomposition of ID feature projections. As shown in Lemma.\ref{lemma1}, after ID features-targeted model sparsity, the model performance derivation between ID and OOD data distributions will also be reduced, which is related to the difference in the proportions of different data distributions. Furthermore, since $p_i\gg p_o$, the response of the left singular vector (model with undirected learning) to the ID data is larger than the model's response to the OOD data by a scale of $\frac{p_i}{p_o}$. However, after the effective model sparsity by SFP, the ratio of model's response to ID and OOD data can be reduced to $\frac{\vartheta p_i}{m  p_o}$.
The setting of $\vartheta$ will be discussed later.

\subsection{Correspondence between Model Substructure and Spurious Features}{\label{sec:3.3}}
In this section, we will give out the optimization target of SFP, and then inlustrate the setting of the model sparsity coefficient. We theoretically demonstrate that: with a reasonable setting of the sparse penalty for ID data, SFP can effectively reduce the overfitting of the model on ID features while retaining the learning on invariant features. \par
Specifically, for a training sample $x$ which is identified by SFP as spurious features-dominant ID data, we define $f(x)$ as the last feature maps output by the model and $f(x)$ is also the projection of $x$ into the model space to be learned defined on the spanning set $E$. We use $x\sim F$ to represent $x\in ID$ since $F$ is the basis of the rowspace spanning ID training samples. And for the reason, we use $x\sim G$ to represent $x\in OOD$. We simply use $E^\top F$, $E^\top G$ to denote the projection of the input features in model space.
Thus, the optimization target of SFP can be formulated as follows: 
\begin{equation}   
\begin{aligned}
\min_{E}\mathcal{L}=\mathcal{L}_{ce}+ \eta E_{x\sim F}f(x), 
\end{aligned} 
\label{16}
\end{equation}
where $\eta$ is the sparsity factor imposed on the feature projections for the identified ID data. With that, SFP could adjusts model structures by adaptively recalibrating the channel-wise feature responses of the ID data at a rate $\eta$.
\begin{lemma}
Define $e= |f^*(x)-f(x)|$ as the difference between true feature maps $f^*(x)$ and $f(x)$. When $\eta < 2e$, SFP could effectively reduce the learning of the model to spurious features but keep the performance on the same features.
\label{lemma2}
\end{lemma}
\begin{proof}[Proof of \textbf{Lemma.\ref{lemma2}:} ]
The prediction errors of feature projections $L_{f}$ can be defined as:
\begin{equation}
    \begin{aligned}
L_{f} &= |f^*(x)-f_(x)|^2\\
&=\sum_{i,j=j_1\cup j_2}(f^*(x)
-\sigma_{i,j_1}\xi_{i}\top\lambda_{j_1} -\sigma_{i,j_2}\xi_{i}\top\gamma_{j_2})^2,
\end{aligned}
\label{17}
\end{equation}
and the corresponding gradient is:
\begin{equation}
    \begin{aligned}
\frac{\partial L_{f}}{\partial \sigma_{i,j_1}\xi_{i} } 
&=\frac{\partial e^2}{\partial \sigma_{i,j}\xi_{i}}=2e\frac{\partial e}{\partial \sigma_{i,j}\xi_{i}}\\
&=2e \frac{\left|f^*(x)
-\sigma_{i,j_1}\xi_{i}\top\lambda_{j_1} -\sigma_{i,j_2}\xi_{i}\top\gamma_{j_2})\right |}{\partial \sigma_{i,j}\xi_{i}}\\
&=-2e\lambda_{j_1},
\end{aligned}
\label{18}
\end{equation}
where $i$, $j$ is the index of column vectors in the orthogonal basis for model space and feature space, respectively.
For OOD data, the gradient of the column vectors in the OOD projection matrix interacting with the $j_{th}$ feature vector is $-2e\gamma_{j_2}$.

\noindent Then, splitting the in-domain features into the spurious features $F'$ and the invariant features $IN$, and splitting the out-of-domain features into the unknown features $G'$ and the invariant features $IN$. 
Since the environment features in-domain and out-domain are different with high probability under OOD setting, we suppose $F'$ and $G'$ are orthogonal. 
To achieve spurious features-targeted unlearning and invariant features-targeted learning of the model, we need to satisfy the following constraint:
\begin{equation}
\begin{aligned}
    &2ep_i\lambda_{IN}+2ep_o\gamma_{IN}-p_i\eta \lambda_{IN}>2ep_o\gamma_{G'}\\
    &\Rightarrow~\eta \le \frac{2ep_i\lambda_{IN}+2ep_o\gamma_{IN}-2ep_o\gamma_{G'}}{p_i \lambda_{IN}}\approx 2e.
\end{aligned}
\label{19}
\end{equation}
\end{proof}
Since the de-learning rate of the spurious feature is positively correlated with $\eta$, the upper bound $\eta=2e$ is taken in this work.
\bigskip
\setlength{\belowdisplayskip}{3pt}
\section{Experiments}{\label{sec:4}}
In this section, we conduct extensive experiments on various datasets to demonstrate the effectiveness and superiority of the proposed SFP. We first compare the average loss of ID and OOD samples in the training process to prove Lemma.~{SFP} from an experimental perspective. Then, we compare the OOD generalization performance with other state-of-the-art methods.  
\subsection{Experiment Setting}
We evaluate SFP over three OOD datasets: Full-colored-mnist, Colored-object, and Scene-object. As shown in Fig.~\ref{fig: ood dataset}, in these datasets, the invariant feature is the focused digit or object in the foreground, and the spurious feature is the background scene. In the full-colored-mnist and colored-object dataset, we generate ten pure colored backgrounds with different colors as the spurious feature. 
In the scene-object dataset, 
we extract ten real-world scenes from PLACE365 dataset \cite{zhou2017places} as the background. 
Besides, the ten kinds of objects in colored-object and scene-object are also extrated from MSCOCO dataset \cite{lin2014microsoft}.
For all datasets, we design the biased data samples with a one-to-one object-scenery relationship, e.g., in the full-colored-mnist, the biased sample of digit $1$ always has a pure red background, and the unbiased sample has a background with a randomly assigned color. 
The former kind of sample is considered an ID sample since it contains spurious features, while the latter kind is considered an OOD sample in our experiment.

\begin{figure}[t]
\centering
\setlength{\abovecaptionskip}{1pt}
\begin{subfigure}[t]{0.15\textwidth} 
    \includegraphics[width=1in]{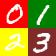}
    \caption{full-colored-mnist}
\end{subfigure}
\begin{subfigure}[t]{0.15\textwidth}
    \includegraphics[width=1in]{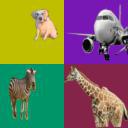}
    \caption{colored-object}
\end{subfigure}
\begin{subfigure}[t]{0.15\textwidth}
    \includegraphics[width=1in]{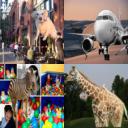}
    \caption{scene-object}
\end{subfigure}
\caption{Visualization of OOD datasets}
\label{fig: ood dataset}
\end{figure}

\begin{figure}[t]
\label{performance result}
\centering
\setlength{\abovecaptionskip}{1pt}
\includegraphics[width=0.35\textwidth]{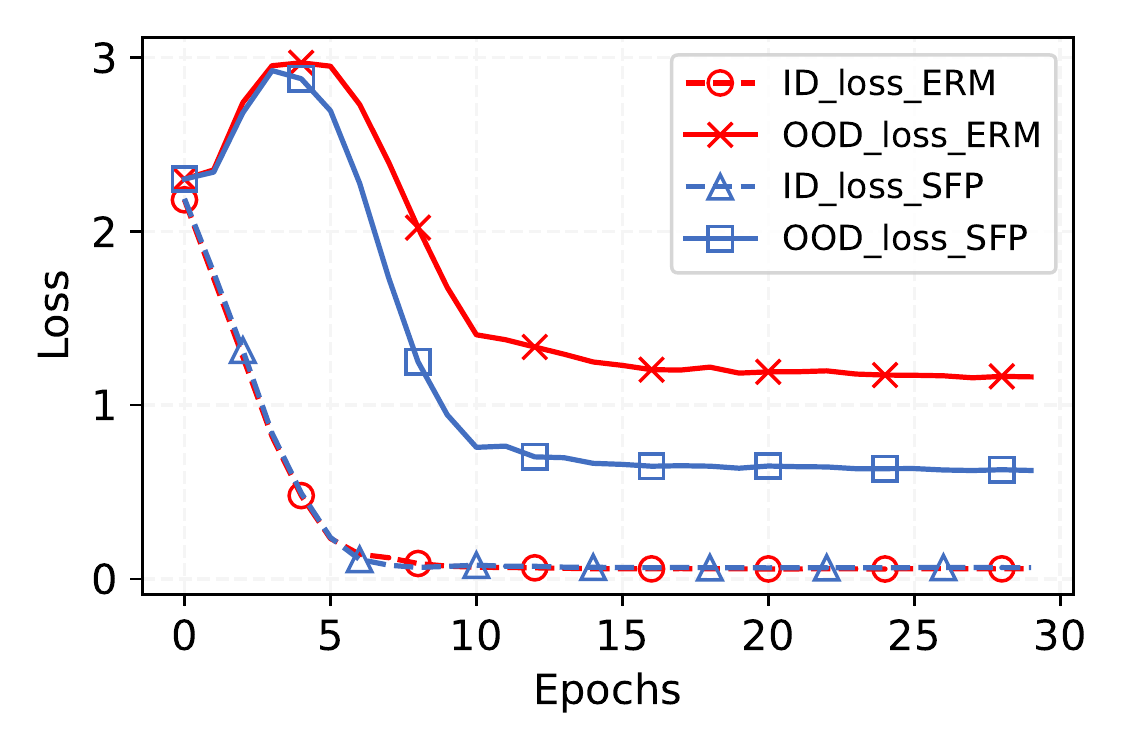}
\caption{Training loss of ID and OOD samples}
\label{fig: loss curve}
\end{figure}

\begin{table}[t]
\centering
\setlength{\abovecaptionskip}{2pt}
\caption{Performance comparison of SFP and OOD generalized SOTAs on full-colored-mnist.}
\label{table1}
\begin{tabular}{c|cc}
\toprule
Method & Train Accuracy(\%)       & Test Accuracy(\%)   \\ 
\midrule\midrule
ERM                     & $93.96$        & $62.20$          \\
MRM                     & $96.71$        & $80.95$         \\
SFP                     & $97.48$        & $84.29$        \\ 
\midrule
IRM                     & $96.25$        & $77.96$         \\
MODIRM                  & $98.11$        & $89.32$        \\
SFPIRM                  & $98.35$        & $89.93$        \\ 
\midrule
REX                     & $97.44$        & $87.80$         \\
MODREX                  & $98.39$        & $92.19$        \\
SFPREX                  & $98.61$        & $93.42$        \\ 
\midrule
DRO                     & $94.05$        & $62.89$        \\
MODDRO                  & $96.73$        & $80.52$       \\
SFPDRO                  & $97.56$        & $85.24$         \\ 
\midrule
UNBIASED               & $93.36$        & $94.04$            \\ 
\bottomrule
\end{tabular}
\end{table}

\begin{table}[h]
\centering
\setlength{\abovecaptionskip}{2pt}
\caption{Performance comparison of SFP and OOD generalized SOTAs on colored-object.}
\label{table2}
\begin{tabular}{c|cc}
\toprule
Method & Train Accuracy(\%)      & Test Accuracy(\%)       \\ 
\midrule\midrule
ERM                        & $99.99$        & $59.19$       \\
MRM                       & $99.99$        & $60.65$        \\
SFP                     & $100$          & $61.01$         \\ 
\midrule
IRM                        & $99.99$        & $62.88$      \\
MODIRM                   & $99.98$        & $64.52$         \\
SFPIRM                  & $100$          & $65.8$       \\ 
\midrule
REX                    & $100$          & $64.72$      \\
MODREX                    & $100$          & $64.52$    \\
SFPREX                      & $100$          & $66.08$    \\ 
\midrule
DRO                      & $100$          & $66.76$      \\
MODDRO                   & $100$          & $66.18$       \\
SFPDRO                   & $100$          & $68.44$       \\ 
\midrule
UNBIASED                & $99.98$        & $75.78$      \\ 
\bottomrule
\end{tabular}
\end{table}

\begin{table}[h]
\centering
\setlength{\abovecaptionskip}{2pt}
\caption{Performance comparison of SFP and OOD generalized SOTAs on scene-object.}
\label{table3}
\begin{tabular}{c|cc}
\toprule
Method & Train Accuracy(\%)       & Test Accuracy(\%)      \\ 
\midrule\midrule
ERM                   & $100$         & $27.4$       \\
MRM                     & $100$         & $26.74$      \\
SFP                      & $100$         & $28.41$      \\ 
\midrule
IRM                      & $99.79$       & $36.88$      \\
MODIRM                & $99.78$       & $36.92$      \\
SFPIRM                    & $99.76$       & $38.1$        \\ 
\midrule
REX                       & $99.76$       & $36.71$      \\
MODREX                     & $99.82$       & $36.66$     \\
SFPREX                    & $99.68$       & $37.91$      \\ 
\midrule
DRO                         & $99.98$       & $31.31$      \\
MODDRO                      & $99.85$       & $29.38$      \\
SFPDRO                      & $99.98$       & $31.78$      \\ 
\midrule
UNBIASED                     & $99.85$       & $45.51$        \\ 
\bottomrule
\end{tabular}
\end{table}
\subsection{Loss Tracking}
\setlength{\belowdisplayskip}{2pt}
To verify the efficiency of our proposed regularization term in SFP, we visualize the changes of loss value on ERM and our SFP with the increase of training rounds. 
%
We can see from Fig.~\ref{fig: loss curve} that the loss of the ID samples is always lower than the loss of the OOD samples in the whole training process, which positively verify the result of Proposition.~\ref{proposition3}, i.e., SFP can effectively filter the OOD samples by the task loss.
Besides, We can notice that 
the loss of ID samples converges too fast in the ERM method (i.e., the red lines in Fig.~\ref{fig: loss curve}),  while the OOD samples remain large losses. The ERM model pays too much attention to the biased data, so they tend to fit the spurious feature and ignore the invariant feature. 
On the contrary, in SFP, the distance between ID samples' and OOD samples' loss is significantly reduced, so making the ID samples' features sparse is helpful to keep the model paying more attention to the OOD samples. More importantly, as we can see in Fig.~\ref{fig: loss curve}, the regularization term of our SFP neither slow down the convergence speed nor negatively influence the performance of ID samples.

\subsection{OOD Generalization}

In this section, we compare the ood generalization ability of our proposed SFP and other state-of-the-art baselines, which contains three OOD generalization techniques: IRM \cite{arjovsky2019invariant}, REX \cite{krueger2021out}, DRO \cite{sagawa2019distributionally}, and one structure-learning method, MRM \cite{can}. 
Since the MRM and our proposed SFP are both orthogonal to the other three baselines, we also integrate SFP into them to 
to compare performance promotion.

To evaluate the OOD generalization capability, we train the model using two in-domain environments dominated by biased samples and evaluate the performance in an OOD environment. By such means, the upper bound of the OOD generalization performance can be achieved by training the model in an environment with only unbiased samples.
Supposing a biased ratio coefficient $(0.8,0.6,0.0)$ represents the ratio of biased data in two training environments and one testing environment, we set different biased ratio coefficients for different datasets. For full-colored-mnist and colored-object, we set the biased ratio coefficient as $(0.8,0.6,0.0)$. To increase the difficulty in the scene-object, we set the biased ratio as $(0.9,0.7,0.0)$. The Unbiased performance is tested in environment $(0.0,0.0,0.0)$.
The OOD generalization results over three datasets are illustrated from Table~\ref{table1} to Table~\ref{table3}.  We can notice that our proposed SFP can effectively improve the OOD generalization performance in all cases. The most significant case is the performance in full-colored-mnist task cooperates with REX algorithm, who reaches a high accuracy of $93.42\%$. By contrast, the upper bound accuracy of UNBIASED method is $94.04\%$, only surpassing our method $0.62\%$.\par
Even though MRM's promotion is adaptable to other state-of-the-art baselines, the benefit of MRM is unstable. 
In the full-colored-mnist task, MRM and SFP can assist other domain generalization algorithms, while SFP's promotion is higher. However, in more complex tasks such as colored-object and scene-object tasks, MRM sometimes has a negative effect. For example, in the scene-object task, the test accuracy of the DRO algorithm can achieve $31.31\%$ by itself. With the integration of MRM, the performance is dragged down to $29.38\%$, while SFP helps to increase the accuracy to $31.78\%$.  
\section{Conclusion}
In this paper, we propose a novel spurious feature-targeted model sparsity framework, dubbed SFP, to automatically find out the substructures with better OOD generalization performance that the original model. 
By effectively identifying spurious features during training via a theoretically verified threshold, SFP can perform targeted-model pruning that removes model branches only with strong dependencies on spurious features. Thus, SFP attenuates the projections of spurious features into model space and pushes the model learning toward invariant features, devising the substructure with the optimal OOD generalization. 
We also conduct detailed theoretical analysis to provide the rationality guarantee and a proof framework for OOD structures via model sparsity.
To the best of our knowledge, this is the first work to theoretically reveal the correspondence between the biased data features and the model substructures learning with respect to better OOD generalization.
\section{here}
\begin{algorithm}[tb] 
\caption{ The training pipeline of the proposed spurious feature-
targeted model pruning. } 
\label{alg1:Framwork} 
\begin{algorithmic}[1] 
\REQUIRE ~~\\ 
Training dataset $X=X_{id}\cup X_{ood}$, wherein $X_{in}$ and $X_{ood}$ are the in-distribution and out-of-distribution dataset respectively;\\
$p_i$ and $p_o$ are the proportion of $X_{id}$
and $X_{ood}$ in the training dataset respectively, and $p_i \gg p_o$;\\
Unpruned model $M$ with parameters set $\mathcal{W}$;\\
Average, maximum, minimum of losses in the $t$-th iteration $\overline{\mathcal{L}^t}$, $\mathcal{L}_{max}^t$ and $\mathcal{L}_{min}^t$;
\ENSURE ~~\\ 
The pruned sub-model for out-of-distribution generalization;
\STATE Initialize $t=0$, $\mathcal{W}$;\\
\STATE Randomly draw a batch of samples from $X$;
\STATE $\overline{\mathcal{L}^t}$ $\gets$ $t$-th model forward propagation;
\STATE \textbf{repeat}\\
\STATE ~~~~$t = t+1$;
\STATE ~~~~$\mathcal{L}_{\max}^t$, $\mathcal{L}_{\min}^t$, $\overline{\mathcal{L}^t}$ $\gets$ $t$-th forward pass;
\STATE ~~~~$\sigma_{F^\top G}^t$ $\gets$ $\max(\Sigma_{F^\top G}) = 1-\frac{\mathcal{L}_{max}-\mathcal{L}_{min}}{p_i^2-p_o^2}$ (EQ.~\ref{7});
\STATE ~~~~Threshold $\Delta=p_o(p_i-p_o)(1-\sigma_{F^\top G}^t)$\\
\STATE ~~~~\textbf{for} each sample $(x,y)$:
\STATE ~~~~\quad \textbf{if} $\mathcal{L}(\mathcal{W},x,y)^t- \mathcal{L}^{t-1} < \Delta$ \textbf{then}\\
\STATE ~~~~\quad\quad Add punish $\mathcal{L}^t = \mathcal{L}^t + f(x)$;\\
\STATE~~~~\quad\quad Update model $\mathcal{W} \gets$ $t$-th backward pass;\\
\STATE ~~~~\quad \textbf{end}
\STATE ~~~~\textbf{end}
\STATE \textbf{until} convergence of $\mathcal{W}$
\end{algorithmic}
\end{algorithm}

\bibliography{aaai23}
\end{document}